\documentclass[12pt]{alt2020}

\usepackage{times}
\usepackage{booktabs}
\usepackage{enumerate}
\usepackage[utf8]{inputenc} 
\usepackage[T1]{fontenc}    
\usepackage{hyperref}       
\usepackage{url}            
\usepackage{amsfonts}       
\usepackage{nicefrac}       
\usepackage{microtype}      
\usepackage{amsmath,amssymb}
\usepackage{placeins}

\usepackage{todonotes}
\newcommand{\AM} [1]{}

\def\CC{\mathcal{C}}
\def\D{\mathbb{D}}
\def\Gr{\mathcal{G}}
\def\Bernoulli{\mathrm{Bernoulli}}
\def\distrib{\sim}
\def\E{\mathbb{E}}
\def\Pr{\mathrm{Pr}}
\def\hmu{\hat{\mu}}
\def\N{\mathbb{N}}
\def\UA{\mathcal{UA}}
\def\PA{\mathcal{PA}}

\def\parsec{\par\noindent}

\def\med{\medskip\parsec}
\def\eps{\varepsilon}

\title[Testing Dynamic Network Models]{Toward Universal Testing of Dynamic Network Models}  

\altauthor{%
    \Name{Abram Magner} \Email{amagner@albany.edu} \\
    \addr Department of Computer Science \\ 
        University at Albany, SUNY \\
        Albany, NY, USA \\
    \AND
    \Name{Wojciech Szpankowski} \Email{spa@cs.purdue.edu} \\
    \addr Department of Computer Science \\
          Purdue University \\
          West Lafayette, IN, USA
}

\begin{document}

\maketitle

\begin{abstract}
    Numerous networks in the real world change over time,
    in the sense that nodes and edges enter and leave
    the networks.  Various dynamic random graph models have been 
    proposed to explain the macroscopic properties of these systems and 
    to provide a foundation for statistical inferences and predictions.
    It is of interest to have a rigorous way to determine how well
    these models match observed networks.
    We thus ask the following \emph{goodness of fit} question: 
    given a sequence of observations/snapshots of a growing random 
    graph, 
    along with a candidate model $M$, can we determine whether the
    snapshots came from $M$ or from some arbitrary
    alternative model that is well-separated from $M$ in some natural 
    metric?  We formulate this problem precisely and boil it down to
    goodness of fit testing for graph-valued, infinite-state Markov 
    processes and exhibit and analyze a universal test based on 
    \emph{non-stationary sampling} for a natural class of
    models.
\end{abstract}


\section{Introduction}
\label{sec:intro}

Time-varying networks abound in various domains.  To explain their 
macroscopic properties (e.g., subgraph frequencies, diameter, degree
distribution, etc.) and to make predictions and other inferences 
(such as link prediction, community detection, etc.), 
a plethora of generative models have been proposed (see \cite{Newman,hofstad-book} for broad 
overviews of classical random graph models; for a mathematical approach via
exchangeable random structures, see, for example,
\cite{Veitch2016SamplingAE,orbanzbayesianmodels}).  For scientific 
reasons it is of interest
to be able to judge how well the models (which postulate particular 
mechanisms of growth) reflect reality. 
This leads naturally to the
problem of \emph{goodness of fit testing} for dynamic network 
mechanisms.  At a high level, this entails determining, in some formal
way, whether or not a given observed sequence of graph snapshots is
more likely to have originated from some candidate growth mechanism
or, alternatively from a sufficiently different one.

\paragraph{Related work}
Regarding related work, goodness of fit testing is a classical topic
in statistics, but the focus is typically on real-valued and 
categorical data over small alphabets 
\cite{Diakonikolas,wasserman-statistics}, where the input to the
problem is a sequence
of independent, identically distributed (iid) random variables from an
unknown probability distribution.  This is in contrast with the setting
of dynamic graph models, where what is given is a trajectory (extremely
non-iid) from an unknown model, and each element of the trajectory is
a graph, which potentially shares a large amount of structure with
its temporal neighbors.

Closer in spirit to our work is the (rather recent) literature on
testing of static random graph models \cite{two-sample-tests}.  Here (in the one-sample
case), what is given
is a sequence of iid graphs (\emph{not} a trajectory of an evolving 
graph) coming from an unknown model, and the task is to distinguish
the unknown model from some candidate distribution.  In works so far,
the distribution class under consideration is a particular parametric
family or satisfies strong independence assumptions (e.g., in \cite{two-sample-tests}, 
all edges are assumed to be independent), which is inherently 
not dynamic.

There has been work on testing of dynamic graph models,
such as \cite{BarabasiMeasuringPA}, 
but most such methods have not come with rigorous guarantees and
have proposed tests tailored to particular parameteric distributions (so they are more properly considered to deal with parameter estimation than goodness of fit testing).  In contrast, we formulate our results
in terms of classes of models that we prove to be
distinguishable.

With regard to the specific types of models of interest, our focus in the
present paper is on model classes that are given by inequalities that
control the rate of change of the conditional distributions governing
the evolution of the sample graphs with respect to time.  This affords
us some flexibility in defining tractable model classes.  There exist
alternative theories on which a model class could be based, such as the theory
of graphexes for sparse exchangeable random graphs (described in
\cite{BorgsChayesSamplingPerspectives}).  To a graphex is associated 
a growing random graph model, and the problem akin to ours in that setting
is to produce consistent estimators of the graphex from a graph trajectory.
This has been done in~\cite{Veitch2016SamplingAE}.  However, while the
set of models that can be parameterized by a graphex is quite rich, it 
has limitations that make it unsuitable for some of the models that
we wish to consider, such as preferential attachment and uniform attachment
(according to results in \cite{BorgsChayesSamplingPerspectives}, both of
these converge to the same degenerate graphex).
We mention also \cite{Ryabko2017HypothesisTestingInfiniteRGs}, which deals with
hypothesis testing on infinite random graphs.  The focus there is on testing
properties of infinite, bounded-degree graph \emph{isomorphism classes} based on random
walk samples; in contrast,
our concern is with graph trajectories, which fundamentally carry more information
than isomorphism classes, and our sampling model is different.

Additionally, one of the goals of our paper is to analyze the particularly
natural scheme of non-stationary sampling (which was inspired by the algorithm sketched, but not fully specified, in \cite{BarabasiMeasuringPA}), and the guarantees on it lend
themselves to phrasing in terms of the sequence of conditional distributions
of a model, as we have done.

Finally, we mention recent work on goodness of fit testing for Markov
chains \cite{daskalakis-colt}.  In that work, as in ours, what is given is a single
trajectory from a Markov chain.  However, they restrict to a chain of
constant size and assume homogeneous transition probabilities, 
ergodicity, symmetry, etc.  See also more recent work that drops the symmetry condition~\cite{Cherapanamjeri2019TestingSM,Wolfer2019MinimaxTO}.
General hypothesis testing problems relating to finite
Markov chains under these assumptions have been considered in the more
distant past~\cite{natarajan-large-deviations-mc,Nakagawa1993OnTC,baum1966}.
In the setting of our paper, none of these assumptions hold.

\paragraph{Our contributions}
In this paper, we give a rigorous formulation of goodness of fit 
testing for dynamic random graph models satisfying a Markov property, 
including a natural distortion measure on dynamic mechanisms.  We identify (at an 
intuitive level) the key 
properties of dynamic models that make universal testing (i.e., testing with provable guarantees for a broad range of models) feasible and, 
motivated by these, propose a test of goodness of fit for 
models in a natural, infinite class
based on \emph{non-stationary sampling}, 
and show 
that this test succeeds with high probability.  Our work may also be viewed
as contributing toward the theoretical understanding of the generality of
the idea of non-stationary sampling.  Intuitively speaking, we find that
it is not \emph{completely} general, because of pathological special cases,
but we can carve out a class of models defined by very natural conditions for 
which this technique is useful.

There are several novel technical challenges that arise in the 
problem of testing of dynamic random graph models.  In particular,
the problem deals with distinguishing between a candidate model
and an arbitrary unknown model from a model class, both of which 
take the form of Markov processes on an infinite state space, so
that classical tools for finite-state, ergodic Markov chains do
not apply.  Furthermore,
identification of appropriate metrics for measuring the distance
between two dynamic models is a nontrivial philosophical problem:
we argue that, for example, total variation distance is of limited
applicability in this setting, because it is intuitively substantially too stringent for exploratory data analysis: one may consider two models that
are, intuitively, driven by the same mechanism (e.g., preferential attachment), but that have total variation distance $1$ because
of small model differences.

Having settled on some measure, identification of an
appropriate model class under which our candidate test can be shown
to succeed with high probability is nontrivial, due to the large
number of potentially pathological models that need to be ruled out.

The initial inspiration for the present work was to explore principled 
techniques for model validation in order to explore the validity of the application of
the algorithms in \cite{scientific-reports} to real networks, such as
protein interaction and brain region co-excitation networks.  In that
work, the authors devised algorithms and information-theoretic bounds to estimate
the arrival order of nodes from a snapshot of a growing graph, assuming
preferential attachment as a generative model.  We regard the present paper
as forming initial groundwork for subsequent developments along the lines
of model validation.

\subsection{Organization of the paper}
In the main body of the paper, we state the problem formally and give
main definitions and results and experiments.  We give the proofs in 
Appendix~\ref{ProofAppendix}.

\section{Formal problem setting}
We begin by formulating the problem in general.  The most basic object
of study is a (discrete-time) \emph{dynamic random graph model}.

\begin{definition}[Dynamic random graph model]
    A dynamic random graph model (or dynamic mechanism) is a 
    graph-valued Markov process.  More specifically, it is specified 
    by a sequence of conditional distributions $\Pr[G_0]$, 
    $\{\Pr[G_t | G_{t-1}]\}_{t=1}^\infty$. 
\end{definition}
The Markov assumption is a natural one, especially in light of the
fact that many dynamic graph models satisfy it (including the various
preferential attachment models, the Cooper-Frieze web graph model \cite{Newman}, the
dynamic stochastic block model, the duplication-divergence model, etc).  It has 
additionally already appeared explicitly in works on dynamic graphs
\cite{AvinLotker}.  However, one can imagine plausible situations in which it
does not hold: for example, nodes may be \emph{nostalgic}, in the sense
that they may choose to connect with others
that were, at some point in the past, of high degree, or with nodes
to which they had been previously frequently connected.  We do not
deal explicitly with such situations in the present work (but
it is possible that they may be accommodated by a suitable change
in the definition of the network under consideration).

We next define a distortion measure on Markov processes, which will be used
to define our testing problem.
\begin{definition}[Distortion measure on Markov processes]
	Consider two Markov processes 
    $X_0 = (X_{0,1}, ..., X_{0,n}, ...)$, $X_1 = (X_{1,1}, ..., X_{1,n}, ...)$.  
    We define the following distortion:
    \begin{align}
    	d_n(X_0, X_1) 
        = 
        \sum_{j=1}^{n-1} \E_{Z\sim X_{1,j}}[d_{TV}([X_{0,j+1} | X_{0,j}=Z], [X_{1,j+1} | X_{1,j} = Z])],
    \end{align}
where $d_{TV}$ is the total variation distance between the laws of the two
random variables.
\end{definition}    
We note that the proposed distortion measure is not symmetric, which is natural in our problem setting. 
In general, the choice of distortion measure depends on the application at hand.
It implicitly encodes which models we choose to regard
as distinguishable, and in subsequent work we intend to examine 
this issue more closely.  The measure under consideration here has the following
natural interpretation, coming from the interpretation of the total variation
distance: consider a setting in which we observe a sample $Z$ from 
$X_1$ at the time step $t-1$, for a uniformly randomly selected time 
$t \in [n]$.  A fair coin $B \sim \Bernoulli(1/2)$ is then flipped, and we
then observe a sample from $X_{B,t}$, conditioned on $X_{B,t-1}=Z$.  We
are then asked to guess the value of $B$ (i.e., which process generated the
next value conditioned on the initial value that we saw).  Up to constant
factors, $d_n(X_0, X_1)/n$ gives the success probability of an \emph{optimal}
hypothesis test for this problem (this is folklore about
total variation distance).
In other words, \emph{$d_n(\cdot, \cdot)$ measures the 
distinguishing power of a 
uniformly random timestep from a trajectory from one of the 
models.}  See Example~\ref{DistortionExample} below for an 
illustration.

\paragraph{Problem statement:}
Having introduced dynamic mechanisms and a measure of distortion between them,
we come to our main problem statement.  Fix a particular natural 
class $\CC$ of models with some distinguished member $M_0$.  Let
$\Gr$ denote the set of all graphs (for simplicity, let us 
throughout only consider multigraphs with finitely many vertices 
and edges).  Let $B \distrib \Bernoulli(1/2)$.
We wish to exhibit a \emph{test} for $M_0$, 
which takes the form
of a function $F=F_{M_0}:\Gr^n \to \{0, 1\}$ mapping sequences of
graphs to $\{0, 1\}$; that is, the input to the test $F$ is an 
observed dynamic graph trajectory with $n$ timesteps, 
and the output represents 
whether or not the test decides that the trajectory came from 
$M_0$.  Ideally, the test should distinguish trajectories
coming from $M_0$ from those coming from an arbitrary unknown 
$M_1 \in \CC$
satisfying $d_n(M_0, M_1) \geq n\epsilon$ (we will find that
$\epsilon$ cannot always be taken to be arbitrarily close to $0$,
due to an intrinsic property of a model which we call its
non-stationary sampling radius).  
More
precisely, $F$ takes $n$ steps $G_1, ..., G_n$ from $M_B$ and
should satisfy $F(G_1, ..., G_n) = B$ with high probability (asymptotically as $n\to\infty$).
In what follows, we will exhibit
a class $\CC$ and a universal test $F$ for it that succeeds with high
probability under natural conditions.  

The class will be
inspired by (and will contain) the (linear)
\emph{preferential 
attachment} model, which has received extensive attention over several 
years in various communities
as a simple model that produces a power law degree distribution 
\cite{barabasi-albert,bollobas-riordan}.  
This model, denoted by $\PA(m, n)$, is defined as follows: there is a
parameter $m \in \N$.  The graph $G_1$ is a single vertex (called $1$)
with $m$ self loops.  Conditioned on $G_{t-1}$, $G_t$ is sampled as
follows: we have $G_{t-1} \subset G_{t}$, and there is an additional 
vertex $t$ with $m$ edges to vertices in $[t-1] = \{1, ..., t\}$,
chosen independently as follows: the probability that a particular
choice goes to vertex $v \in [t-1]$ is given by 
$\frac{\deg_{t-1}(v)}{2m(t-1)}$,  where $\deg_{t-1}(v)$ is the degree
of $v$ in $G_{t-1}$.  There are several small tweaks to this model,
but they have no bearing on the analysis in this paper.  In examples, we will also mention the \emph{uniform attachment} model, which is similar to preferential attachment, except
that the conditional probability of a given vertex $v$ being
chosen at time $t$ is $1/(t-1)$.  We denote this model by
$\UA(m, n)$.

\begin{example}
    \label{DistortionExample}
    Consider the case of uniform attachment versus preferential attachment graphs with
    shared parameter $m\geq 1$.  
    Let us consider $d_n(\PA, \UA)$.  The $j$th term of
    the sum defining the distortion measure is an expectation
    with respect to a random variable $Z \sim \UA(m, j)$; that is,
    $Z$ is a uniform attachment graph on $j$ vertices.  Now,
    the conditional probability distribution of each of the choices of
    the $j+1$st vertex, given
    $Z$, in the uniform attachment model is simply the uniform
    distribution on the vertices $[j] = \{1, ..., j\}$.  Meanwhile,
    the conditional distribution of each choice made by the $j+1$st vertex, given
    $Z$, in the preferential attachment model depends on the
    degrees of vertices in $Z$.  It is known that with high
    probability in the uniform attachment model, as $j\to\infty$,
    there are $\Theta(j)$ vertices with degree $m$, and so
    the preferential attachment distribution assigns $\Theta(j)$
    vertices a conditional probability equal to 
    $\frac{1}{2(j-1)}$.
    In particular, this implies that $d_n(\PA, \UA) = \Omega(n)$.
    A more careful analysis, using known results from the 
    literature, can yield precise asymptotics.
\end{example}
%

\section{Main results}
We will build slowly to our main results and model class definition.
The main difficulty in the problem is that we are tasked with inferring
information about a sequence of probability distributions, usually given only a
\emph{single} (or few) graph trajectory, which implies only a single sample of each 
conditional distribution.  In general, then, we must appropriately
restrict our model class $\CC$ in order to yield a solvable problem.
The key insight is that, between consecutive time steps, the network
structure, and, hence, the corresponding conditional distributions, 
should not change significantly (note, however, that this is a 
nontrivial phenomenon to formalize).  Our plan of attack will be to
\emph{pretend} that the sequence of updates to the graph immediately after
a given time point are independently sampled from the same distribution
and to construct an estimator of our metric from these samples.  
For simplicity of exposition, we restrict below to the case of growing 
networks, where a single vertex is added at each timestep, and it connects 
to some set of previous vertices.  In this setting, an \emph{update} (with 
respect to a given graph sequence) at time $t$
takes the form of a selection (a multiset) of vertices present in the graph at 
the previous timestep to which vertex $t$ will connect.  We will denote the
set of possible updates at time $t$ by $\D_t$.  For a given update $\kappa \in \D_t$, we will use $t\to \kappa$ to denote the event that vertex $t$ connects to
the vertices specified by $\kappa$ at time $t$.  We denote by $U_t$ the update
at time $t$ in a given sample graph trajectory.

Formally, we will define the test statistic as follows.  
Fix $C(n)$ (which we call the \emph{sample width}) and $M(n)$ (the 
\emph{number of probe points}), and, for $t\in[n]$ and $\kappa \in \D_t$, let
$\hmu_{t,\kappa}$ denote the following 
probability:
\begin{align}
    \hmu_{t,\kappa}
    =
    \frac{| \{h \in [t, t+C(n)) ~:~ h\to \kappa \}|}
         {|\{ h \in [t, t+C(n)) ~:~  U_h \subseteq [t-1]   \}|}.
\end{align}
Let $\hmu_t$ denote the probability distribution giving probability $\hmu_{t,\kappa}$ to $\kappa$,
for each possible $\kappa$.
That is, $\hmu_{t}$ is an empirical probability distribution over updates.  We
note that the sample intervals corresponding to different probe points in general
can overlap substantially. 

Let $p^{M_0}_{t,\kappa} = p_{t,\kappa}$ denote the conditional probability given
to update $\kappa$ at time $t$ by model $M_0$.  For example, when $M_0$
is preferential attachment with parameter $m \geq 1$ and 
$G_{t-1}$ is given, $\kappa$ is a multiset $\{v_1, ..., v_m\}$ of cardinality $m$ of
vertices in $G_{t-1}$ (in particular, those to which the
new vertex $t$ will connect), and $p_{t,\kappa}$ is given by
\begin{align}
    p_{t,\kappa} 
    = \prod_{i=1}^m \frac{\deg_{t-1}(v_i)}{2m(t-1)}.
\end{align}
Given an input trajectory $G = (G_1, ..., G_n)$, the test 
statistic is computed by randomly 
sampling $M(n)$ \emph{probe points} 
$r_1 \leq ... \leq r_{M(n)} \in [n]$ and computing
\begin{align}
    S(n) := S_{M_0}(G_1, ..., G_n)
    = \sum_{j=1}^{M(n)} d_{TV}(\hmu_{r_j}, p^{M_0}_{r_j}).
    \label{SDefinition}
\end{align}
We often shall write $S:=S(n)$ for $S_{M_0}(G_1, ..., G_n)$.
This may be thought of as an estimate of $d_{n}(M_B, M_0)$ via 
non-stationary sampling. 
Now we are ready for present our test.  \emph{We will further restrict to the
case where the cardinality of an update is at most some fixed $m \geq 1$.}
This is natural in light of the fact that many real networks tend to be sparse.

\begin{center}
\framebox[5.5in]{\parbox{5.4in}{
{\bf Algorithm}: {\sc Test-DynamicGraph} $(n, M_0, M(n), C(n))$.
\med
Input: 
Sample trajectory $G=(G_1, \ldots, G_n) \sim M_B$, distance threshold $D \geq 0$. \\
Output: An estimate $\hat{B} \in \{0, 1\}$ of $B$. 

\begin{enumerate}
\item Select $M(n)$ probe points $r_1, \ldots, r_{M(n)}$ uniformly at random
with replacement from $[n]$
and compute $\hmu_{r_j}$ for each $j \in [M(n)]$.

\item Compute
$$
S(n):=S_{M_0}(G_1, \ldots, G_n)=\sum_{j=1}^{M(n)} d_{TV}(\hat{\mu}_{r_j}, p_{r_j}^{M_0}).
$$

\item Set $\alpha = \alpha(n) = Dn/2$, and compute $\E_{M_0}[S_{M_0}]$,
where $\E_{M_0}[S_{M_0}]$ can be estimated by sampling  
$G'=(G_1', \ldots, G_n')$ from $M_0$ (or, possibly, computed analytically).


\item If $|S(n)-\E_{M_0}[S_{M_0}]| >\alpha:=\alpha(n)$, set $\hat{B}=1$, else $\hat{B}=0$.

\end{enumerate}

}}
\end{center}

We note that although $S$ has the form of an estimate of 
$d_n(M_B, M_0)$, it does \emph{not} converge in probability to this value (indeed, it is likely that a single trajectory is insufficient
to estimate $S$).  However, as we will see, the above test succeeds
with high probability under natural conditions.

To state our main result for this test, we need to develop some
concepts related to our model class (from which both $M_0$ and $M_1$ will come).  The general pattern of
the analysis of the test to show that it succeeds with high probability
in distinguishing two models will proceed in two steps:
\begin{itemize}
    \item
        Show that $|\E_{M_1}[S] - \E_{M_0}[S]|$ is sufficiently large.
    \item
        Show that $S$ is well-concentrated under both $M_0$ and $M_1$.
\end{itemize}

Let us examine the first step more closely.  We can lower bound the
expected value difference as follows: let $S_j$ denote the $j$th
term in the sum (\ref{SDefinition}) defining the test statistic (note that it
is defined with respect to $M_0$).
We have
\begin{align}
    \E_{M_0}[S_j]
    = \E_{M_0}[d_{TV}(\hmu_{r_j}, p_{r_j}^{M_0})],
\end{align}
and
\begin{align}
    \E_{M_1}[S_j]
    = \E_{M_1}[d_{TV}(\hmu_{r_j}, p_{r_j}^{M_0})].
\end{align}
By the reverse triangle inequality, this becomes
\begin{align}
    \E_{M_1}[S_j]
    \geq \left| \E_{M_1}[d_{TV}(\hmu_{r_j}, p_{r_j}^{M_1})]
            - \E_{M_1}[d_{TV}(p_{r_j}^{M_1}, p_{r_j}^{M_0})]
    \right|.
\end{align}
This can be further lower bounded by
\begin{align}
    \E_{M_1}[S_j]
    \geq \E_{M_1}[d_{TV}(p_{r_j}^{M_1}, p_{r_j}^{M_0})]
         - \E_{M_1}[d_{TV}(\hmu_{r_j}, p_{r_j}^{M_1})].
\end{align}
Thus, we have
\begin{align}
    |\E_{M_1}[S_j] - \E_{M_0}[S_j]|
    \geq \E_{M_1}[d_{TV}(p_{r_j}^{M_1}, p_{r_j}^{M_0})]
        - \E_{M_1}[d_{TV}(\hmu_{r_j}, p_{r_j}^{M_1})]
        - \E_{M_0}[d_{TV}(\hmu_{r_j}, p_{r_j}^{M_0})].
\end{align}
The positive term measures the contribution of the point $r_j$
to the distance between $M_0$ and $M_1$.  Meanwhile, the negative
terms are both intrinsic estimation errors from non-stationary 
sampling.
In order for two models to be easy to distinguish, we desire that
these terms be small in absolute value. In the sequel we call
$\E_M[S_M]$ the \emph{non-stationary sampling radius}.   


We will show in Section~\ref{MainTheoremProof} a concentration result for 
the test statistic 
for any $M_0, M_1$ in the model class $\CC$ (inspired by the analysis of
non-stationary sampling on the preferential attachment model) described below
in the following definition.
\begin{definition}[Bounded-degree model class]
    Let $\CC = \CC_m$, for any fixed $m \geq 1$, denote the class of dynamic random graph models $M$
    taking the following form:
    for each time step $t$, there is a positive integer-valued 
    random variable $\Gamma_t \leq m$, independent of $G_{t-1}$
    and all other $\Gamma_{t'}$, denoting
    the number of edges to be
    added at timestep $t$ between a new vertex $t$ and vertices
    present in the graph at the previous timestep.  Conditioned
    on $\Gamma_t$ and $G_{t-1}$, the random variable $U_t$ 
    consists of $\Gamma_t$ independent and identically 
    distributed choices of vertices in $[t-1]$, according to
    a probability distribution $\{\pi_{t,v}\}_{v=1}^{t-1}$
    (note that a vertex may be chosen multiple times in a given
    timestep, which would yield a multigraph). 
    We call $\CC_m$ the class of \emph{bounded-degree} models.
\end{definition}
For instance,
in preferential attachment, $\Gamma_t = m$ with probability $1$, 
and $\pi_{t,v}$ is simply 
$\frac{\deg_{t-1}(v)}{2m(t-1)}$, so that $p_{t,U_t} = \prod_{v \in U_t} \pi_{t,v}$.

We further
stipulate the following conditions.  As will be spelled out 
explicitly in Example~\ref{ModelClassExamples}, these are 
inspired
by the basic properties of the preferential attachment model 
that imply concentration and tail bounds on its degree sequence.

\begin{definition}[Further natural model class conditions]
    \begin{enumerate}[\rm (i)]
        \item
             Each $\pi_{t,v}$ is dependent on $G_{t-1}$ only through a
            random variable $Y_{t,v}$, which, conditioned on $G_{t-1}$,
            is independent of any other $Y_{t',v'}$ with $t' \leq t$ and $v'$.  
            In other words, 
        $G_{t-1} \leftrightarrow Y_{t,v} \leftrightarrow p_{t,v}$
            forms a Markov chain (cf. \cite{coverthomas}).
        \item    
            There is a positive constant $\Delta$ 
            such that, for each $t$, at most $\Delta$ vertices
            satisfy
            $Y_{t,v} - Y_{t-1,v} \neq 0$.  
        \item
            We have $\pi_{t,t-1} = \Theta(1/t)$, uniformly
            in $t$, with probability $1$.
        \item
            Only vertices chosen in a given timestep can
            increase significantly in conditional probability:
            there exist some constants $0 < c_1, c_2$, with
            $c_2 < 1$, such
            that, for any timestep $t$, if a vertex $v$
            is
            not chosen for connection, then 
            \begin{align}
                \pi_{t+1,v} - \pi_{t,v} \leq c_1\pi_{t,v}/t. 
                \label{NotChosenInequality}
            \end{align}
            If, on the other
            hand, $v$ \emph{is} chosen for connection, then
            we only require that
            \begin{align}
                \pi_{t+1,v} - \pi_{t,v} \leq c_2(1-\pi_{t,v})/t.
                \label{ChosenInequality}
            \end{align}    
    \end{enumerate}
    
    \label{ModelClassDefinition}
\end{definition}


\begin{example}
    \label{ModelClassExamples}
    The linear preferential attachment model is easily seen to
    be contained in $\CC=\CC_m$, taking $Y_{t,v} = \deg_{t}(v)$.  The second constraint
    follows for this model because only at most $m$ vertices' degrees change
    after a given time step.  The third constraint follows because
    the degree of vertex $t$ immediately after it is added is $m$,
    the parameter of the model.  That is, 
    $\pi_{t,t-1} = \frac{m}{2m(t-1)} = \frac{1}{2(t-1)} = \Theta(1/t)$.  
    The final constraint can be seen as follows.  If a vertex
    $v$ is not chosen in timestep $t$, then the change in its
    conditional probability at time $t+1$ is given by
    \begin{align}
        \pi_{t+1,v} - \pi_{t,v}
        = \frac{\deg_{t}(v)}{2mt} - \frac{\deg_t(v)}{2m(t-1)}
        = \frac{\deg_{t}(v)}{2m}(\frac{1}{t} - \frac{1}{t-1})
        = -\frac{\pi_{t,v}}{t} \left(1 + O(t^{-2}) \right).
    \end{align}
    On the other hand, if a vertex is chosen (say, exactly
    once) at timestep $t$, then
    \begin{align}
        \pi_{t+1,v} - \pi_{t,v}
        = \frac{\deg_{t-1}(v)+1}{2mt} - \frac{\deg_{t-1}(v)}{2m(t-1)}
        = -\frac{\pi_{t,v}}{t} \left(1 + O(t^{-2}) \right) + \frac{1}{2mt}.
    \end{align}
    
    For another example, consider the \emph{uniform attachment model} with parameter $m$.  In this case, $\pi_{t,v} = \frac{1}{t-1}$ for any $v$.  We may take $Y_{t,v}$ to be
    $\deg_{t-1}(v)$ (though $\pi_{t,v}$ is trivially independent 
    of $G_{t-1}$), so that conditions (i) and (ii) are trivially
    satisfied, as is condition (iii).  Since, in any timestep,
    we have $\pi_{t+1,v} - \pi_{t,v} = \frac{1}{t} - \frac{1}{t-1} = -\frac{1}{t(t-1)}$, condition (iv) is satisfied as well.
    Mixtures of preferential and uniform 
    attachment 
    can additionally be seen to fit into
    this model class.
    
    Additionally, models involving preferential attachment to
    vertices based, e.g., on the number of triangles in which they participate are
    also in $\CC$.
    While many 
    natural models (such as nonlinear preferential attachment and the standard duplication-divergence 
    model) are not obviously contained in this class, similar results to ours (concerning our
    proposed test statistic) nonetheless may be shown to hold with a 
    minor generalization of our model class.  
\end{example}

    We note that conditions (iii) and (iv) together imply an important
    property of models in the class: for a model $M$, let
    $N^{M}_t(q)$ denote the number of vertices $v$ satisfying
    $\pi^{M}_{t,v} = q$ (note that this is a random variable).
    Furthermore, let $R_{M,t}(q)$ denote the following set:
    \begin{align}
        R_{M,t}(q)
        = \{ v ~:~ \pi^{M}_{t,v} \geq q \}.
    \end{align}
    Then conditions (iii) and (iv) together imply that,
    for any $M_0, M_1 \in \CC$,
    possibly with $M_0 = M_1$, we have that
    \begin{align}
        \int_{q}^{1} \E_{M_1}[ N^{M_0}_t(x) ] ~dx
        = \E_{M_1}[|R_{M_0,t}(q)|]
        \leq C (qt)^{1+\gamma},
        \label{ImportantInequality}
    \end{align}
    for some $\gamma < -2$, some $C > 0$, and all $q\in[0, 1]$, 
    $t\leq n$.  Intuitively, the function $N^{M_0}_t(q)$
    is conceptually related to the degree sequence of the graph at 
    time $t$, and, in preferential attachment, it is exactly
    the degree sequence.  Thus, this inequality is akin to a power
    law degree sequence tail.

    Inequality~(\ref{ImportantInequality}) can be seen using the following
    reasoning: we will prove Theorem~\ref{ComparisonTheorem} below,
    which will allow us to upper bound $|R_{M_1,t}(q)|$, for any
    model $M_1$, in terms of
    $|R_{M_0,t}(q)|$, where $M_0$ is some other model.  In particular,
    we will choose $M_0$ to be preferential attachment, and so this
    will allow us to upper bound $|R_{M_1,t}(q)|$ in terms of the
    degree sequence tail of preferential attachment graphs.
    
    \begin{theorem}
        Consider two models $M_0, M_1$ satisfying conditions
        (iii) and (iv) of Definition~\ref{ModelClassDefinition} with
        constants $c_{b,k}$, for $b \in \{0, 1\}$ and $k \in \{0, 1, 2\}$.  
        Suppose,
        further, that $M_0$ satisfies inequalities (\ref{NotChosenInequality})
        and (\ref{ChosenInequality}) with asymptotic equality (as $t\to\infty$), and that $c_{1,k} \leq c_{0,k}$
        for each $k \in \{1, 2\}$.  Then for any sequence of
        vertex choices $v_1, ..., v_t, ...$, we have that there exists some
        $C > 0$ for which
        \begin{align}
            R_{M_1,t}(Cq) \subseteq R_{M_0,t}(q).
        \end{align}
        Therefore,
        \begin{align}
            |R_{M_1,t}(Cq)|
            \leq |R_{M_0,t}(q)|.
        \end{align}
        \label{ComparisonTheorem}
    \end{theorem}
    \begin{proof}
        We note that if all vertices $v$ satisfy
        \begin{align}
            \pi^{M_1}_{v,t}
            \leq C \pi^{M_0}_{v,t},
            \label{ImportantPInequality}
        \end{align}
        then $v \in R_{M_1,t}(Cq)$
        implies that $\pi^{M_1}_{v,t} \geq Cq$, which implies that
        \begin{align}
            \pi^{M_0}_{v,t} \geq q,
        \end{align}
        so that $v \in R_{M_0,t}(q)$, and the inclusion claimed by the theorem
        is verified.
        
        It thus remains to verify (\ref{ImportantPInequality}).
        We can do this by induction on $t$.  At time $t$, when
        vertex $t$ appears, our initial condition says that
        \begin{align}
            \pi^{M_1}_{t,t} = c_{1,0}/t = \frac{c_{1,0} c_{0,0}}{c_{0,0}t}
            = \frac{c_{1,0}}{c_{0,0}} c_{0,0}/t.
        \end{align}
        so we can take $C = c_{1,0}/c_{0,0}$, and this verifies the base case.
        
        Now, for the inductive hypothesis, we have that for any vertex $v$,
        \begin{align}
            \pi^{M_1}_{v,t} \leq C\pi^{M_0}_{v,t},
        \end{align}
        so, for a $v \neq v_{t}$,
        \begin{align}
            \pi^{M_1}_{v,t+1} \leq C\pi^{M_0}_{v,t} + c_{1,1}\pi^{M_1}_{v,t}/t
            \leq C\pi^{M_0}_{v,t+1}
        \end{align}
        as long as $c_{1,1} \leq c_{0,1}$.
        Similarly, as long as $c_{1,2} \leq c_{0,2}$, the same inequality holds
        for $v = v_t$.
    \end{proof}
    To conclude the proof of (\ref{ImportantInequality}), we note that
    the method of proof used to establish upper bounds on the expected
    degree sequence for preferential attachment applies for arbitrary
    fixed values of $c_1 > 0$ and $c_2 \in (0, 1)$.

As hinted above, conditions (iii) and (iv) are 
important in that they imply (\ref{ImportantInequality}).  This
motivates the definition of a more general model class, this
time parametrized by a model $M_0$.
\begin{definition}[More general model class]
    Suppose that $M$ is a bounded-degree (with degree bound $m>0$) model satisfying conditions
    (i) and (ii) of Definition~\ref{ModelClassDefinition},
    in addition to (\ref{ImportantInequality}) for $M_0 = M_1 = M$.
    We define the model class $\CC_{M,m}=\CC_M$ to be the set of
    models $M'$ again satisfying (i) and (ii) of Definition~\ref{ModelClassDefinition}, and also 
    (\ref{ImportantInequality}) with $M_0 = M, M_1 = M'$
    and $M_0 = M', M_1 = M$.
    \label{GeneralModelClassDefinition}
\end{definition}
In what follows, for simplicity, we will say that a pair of models $(M_0, M_1)$ satisfies 
Definition~\ref{GeneralModelClassDefinition} if $M_0$ is as in the definition and $M_1 \in \CC_{M_0}$.

We arrive at our main result.  For a pair of models $M_0$ and $M_1$, it will
be phrased in terms of $d_n(M_0, M_1)$ as well as the sum of their non-stationary sampling radii.  Intuitively, the non-stationary
sampling radius of a model measures the fidelity with which 
non-stationary sampling of a sample trajectory from a model
can estimate the transition probabilities of the model itself.
Thus, to distinguish two models via non-stationary sampling,
it is required that they be well-separated and that they be
estimable via non-stationary sampling.  The expressions in 
the theorem capture this more precisely.

\begin{theorem}[Distinguishability for a natural model class]
    Let $\CC$ be as in Definition~\ref{ModelClassDefinition}.  Let
    $M_0, M_1 \in \CC$ be any pair of models such 
    that 
    \begin{align}
        d_n(M_0, M_1) - \E_{M_0}[S_{M_0}] - \E_{M_1}[S_{M_1}] > Dn, 
        \label{MetricAssumption}
    \end{align}
    where $D$ is the constant input given in the algorithm.
    Alternatively, let $(M_0, M_1)$ satisfy Definition~\ref{GeneralModelClassDefinition},
    again satisfying the inequality (\ref{MetricAssumption}).
    
    Then the test based on the statistic $S$ defined
    in (\ref{SDefinition}), with $C(n)$ and $M(n)$ both chosen to be 
    $\Theta(n)$, succeeds with probability $1 - \Theta(n^{-\delta})$ for
some $\delta=\delta(D)>0$.
    \label{MainTheorem}
\end{theorem}

In order to estimate $\E_{M}[S_{M}]$, 
we will also show that a model's non-stationary sampling radius
can be efficiently estimated from samples, provided that it is
a member of $\CC$.  
\begin{theorem}[Estimability of the non-stationary sampling radius]
    Let $M \in \CC$.  Then the following concentration result holds,
    again in the setting where $M(n)$ and $C(n)$ are $\Theta(n)$:
    for any $c > 0$,
    there exists some $\epsilon > 0$ for which
    \begin{align}
        \Pr_M[ | S_M - \E_M[S_M]| \geq cn] = O(n^{-\epsilon}).
    \end{align}
    \label{EstimationTheorem}
\end{theorem}
This theorem implies that one may estimate the non-stationary
sampling radius of $M$ with high probability up to an arbitrarily small 
multiplicative error by taking a single sample trajectory and computing
$S_M$ on it.

\subsection{Running time and sample complexity upper bounds}
Here we consider the running time of the calculation of the
proposed test.  We stress that the focus of this paper is not on algorithmic
issues, but rather on conditions under which the error probability of the
test can be shown to decay to $0$.  
However, the running time is polynomial
in all relevant parameters, provided that the conditional distributions of $M_0$
can be computed in polynomial time.
\begin{theorem}[Running time of the proposed test]
    Suppose that the probability distribution $p^{M_0}_{t}$ can be computed 
    in time $X(t)$.  Then a naive algorithm for computing $S_{M_0}(G_1, ..., G_n)$ 
    works in time $\Theta(n X(n) + n^2)$.
\end{theorem}
\begin{proof}
    A simple way to perform the computation is to compute each term of
    the sum defining the test statistic independently of any other one.
    There are $M(n) = \Theta(n)$ such terms, and, for each term, computing
    the total variation distance requires at most $X(n) + \Theta(n)$ time.
    The result is that the running time is at most $O(n X(n) + n^2)$.
\end{proof}
As an example, computing the conditional distribution at each timestep $t$
takes time $X(t) = O(t)$ under uniform and preferential attachment models.
Therefore, the running time of the above algorithm for such models is $O(n^2)$.

Regarding sample complexity, we may consider a \emph{sample} to correspond to
a graph at any timestep that we use to compute the test statistic.  Since the
intervals corresponding to different probe points can overlap substantially,
graphs corresponding to certain timesteps may appear several times in the
test statistic calculation.  However, these only count as one sample, and so
we have a sample complexity of at most $O(n)$ (it is likely that this can be
improved).  We stress that our theorems
remain nontrivial and valuable, because they show conditions under which the probability
of error decays to $0$ -- a priori, it need not be the case that this happens,
even with full information about the graph trajectory.

\section{Experimental results}
We empirically investigate the non-stationary sampling radius for 
a certain subset of models.  Since this quantity is crucial for 
our theoretical guarantees,
it is of interest to build intuition about it.

We consider a mixture model interpolating between uniform and preferential
attachment in the case $m=1$.  Namely, at each step, the conditional probability
of vertex $v \in [t-1]$ is given by $\pi_{t,v} = \frac{\beta}{t-1} + (1-\beta)\frac{\deg_{t-1}(v)}{2(t-1)}$, for a parameter $\beta \in [0, 1]$.
Note that $\beta = 0$ corresponds to preferential attachment, while $\beta=1$
corresponds to uniform.

We plotted estimates for the non-stationary sampling radius of several models
as $\beta$ ranges from $0$ to $1$.  The result is in 
Figure~\ref{fig:NonstationarySamplingFig}.  We see a clear linear trend,
with the quantity minimized at $\beta=0$ (preferential attachment).  On
this basis, and given the close relation between the empirical distribution
coming from non-stationary sampling and the degrees of vertices, we have the following conjecture.
\begin{conjecture}
    The model in $\CC$ with the minimum non-stationary sampling 
    radius is preferential attachment. 
\end{conjecture}

\begin{figure}[!htbp]
    \centering
    \includegraphics[scale=0.69]{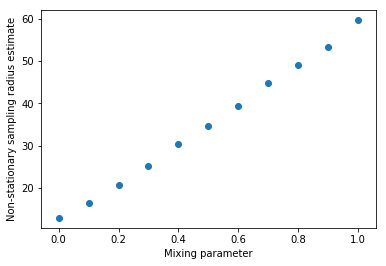}
    \caption{Plot of empirical estimates of $\E_{M}[S_M]$
    for mixtures of uniform and preferential attachment
    with various values of the mixing parameter.  Preferential 
    attachment corresponds to mixing parameter $\beta=0$, while 
    uniform corresponds
    to $\beta=1.0$.  Here, $n=90$. The average is
    over $40$ samples.}
    \label{fig:NonstationarySamplingFig}
\end{figure}

We additionally have results applying our
test to the case of distinguishing uniform attachment (playing
the role of $M_0$) from preferential attachment.  See Figure~\ref{fig:ua-vs-pa}.  The
precision jumps significantly at $n=10$.
\begin{figure}[!htbp]
    \centering
    \includegraphics[scale=0.69]{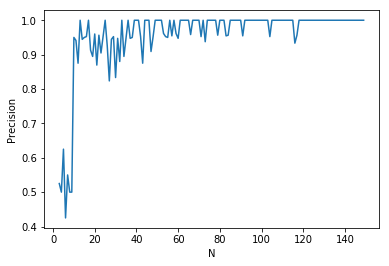}
    \caption{Precision of our test distinguishing uniform
    attachment from preferential as $n$ increases to $150$ (here, precision is the fraction of samples declared to be from preferential attachment, normalized by the number that, indeed, are).  The fraction of trials in which preferential attachment was correctly identified (i.e., the recall) was always $1$.
    For each $n$, $40$ samples
    were taken, with each sample coming from uniform attachment
    with probability $1/2$.}
    \label{fig:ua-vs-pa}
\end{figure}

\FloatBarrier
\bibliographystyle{plain}
\bibliography{dynamic-networks-bib.bib}
\appendix
\section{Proofs}
\label{ProofAppendix}

In this section we prove Theorem~\ref{MainTheorem} and 
Theorem~\ref{EstimationTheorem}.

A roadmap of the proof of Theorem~\ref{MainTheorem} is as follows: we show that, regardless
of which of $M_0$ or $M_1$ generated the observed trajectory, the test statistic $S$ is
well-concentrated around its mean (in both cases, this is typically $\Theta(n)$).  
Furthermore, the hypothesis (\ref{MetricAssumption}) allows us to conclude that 
$\E_{M_0}[S]$ and $\E_{M_1}[S]$ are well-separated.  Thus, we are able to find an appropriate
decision boundary for our test.  The concentration of the test statistic boils down to
concentration of each of its terms, which we are able to establish by rewriting them
as an appropriate double integral in terms of a counting function ($N_t(p, q)$ below)
which is analogous to the degree sequence in preferential attachment.  We are then able
to adapt techniques used to establish concentration of the degree sequence to complete
the proof.  The case for general $m \geq 1$ (where we recall that $m$ is
an upper bound on the number of edges that may be added in a given timestep)
is a corollary of the case for $m=1$, and so below we handle the case of
$m=1$.  This simplifies the notation somewhat: updates are now single vertices, and so we can use $p_{t,v}$ in place of $\pi_{t,v}$, etc.

\subsection{Proof of Theorem~\ref{MainTheorem}}
\label{MainTheoremProof}
The analysis of the test boils down to showing concentration of the
test statistic when the observed graph trajectory is from either $M_0$
or $M_1$.  We will start by proving the following.

\begin{proposition}[Concentration of $d_{TV}(\hmu_{t}, p_{t})$]
    We have that for any models $M_0, M_1$ as in the setting of
    Theorem~\ref{MainTheorem}, there is some $C > 0$ such that 
    for any sufficiently small $\epsilon_1 > 0$, there exists $\epsilon_2 > 0$
    for which
    \begin{align}
        \Pr_{M_1}[ |d_{TV}(\hmu_{t}, p^{M_0}_{t}) - \E_{M_1}[d_{TV}(\hmu_{t}, p^{M_0}_{t})] | > t^{-\epsilon_1} ]
        \leq Ct^{-\epsilon_2}.
    \end{align}
    \label{dtvConcentrationProposition}
\end{proposition}
In order to show Proposition~\ref{dtvConcentrationProposition}, 
we rewrite the total variation
distance as follows: in general, consider discrete random variables $X$ and $Y$.  Then 
\begin{align}
    d_{TV}(X, Y)
    = \frac{1}{2}\int_{0}^1 \int_{0}^1 N(p, q) |p-q| ~dq~dp,
    \label{dtvIntegralRepresentation}
\end{align}
where $N(p, q)$ is the number of elements $x$ in the common domain of $X, Y$
for which $\Pr[X = x] = p$ and $\Pr[Y=x] = q$, and \emph{the integrals are with
respect to discrete (Dirac) measures supported on the set of values for
which $N(p, q)$ can possibly be nonzero}.
We see, then, that concentration
of $d_{TV}(\hmu_{t}, p^{M_0}_{t})$ will follow from concentration
of $N_t(p, q)$ (now parameterized by $t$), where, in our case, the random variables in question are
distributed according to $\hmu_{t}$ (which will correspond to the $p$ integrating 
variable) and $p^{M_0}_{t} = p_{t}$ (which will correspond to the $q$ integrating
variable).  More explicitly, $N_t(p, q)$ is the number of vertices $v$ in the interval
$[t-1]$ whose empirical probability $\hmu_{t,v}$ is equal to $p$ and whose 
conditional probability at time $t$ according to model $M_0$ is equal to $q$.  Note
that $N_t(p, q)$ is a random variable, since it depends on $\hmu_t$, which itself
depends on the observed trajectory.
It is interesting to observe at this point that $N_t(p, q)$ is closely related, in
the preferential attachment case, to the degree distribution, so ideas used in proving concentration
of the degree sequence in that model can be generalized to prove concentration for $N_t(p, q)$.
The next lemma gives the necessary concentration result.
\begin{lemma}[Concentration for $N_t(p, q)$]
    Consider a model $M_1 \in \CC$.  Then 
    \begin{align}
        \Pr_{M_1}[ |N_t(p, q) - \E_{M_1}[N_t(p, q)]| \geq x]
        \leq 2\exp\left( -\frac{x^2}{c\cdot C(n)}\right),
    \end{align}
    for arbitrary $p, q \in [0, 1]$, some positive constant $c$, and 
    we recall that $C(n)=\Theta(n)$.
    \label{NpqConcentrationLemma}
\end{lemma}
To prove this, we will use the method of bounded differences.  Define
the martingale $Z_s$, for $s  \in \{t, t+1, ..., t+C(n)\}$, as
\begin{align}
    Z_s = \E_{M_1}[ N_t(p, q) | G_{s}].
\end{align}
The next lemma establishes the necessary bound on the martingale
differences.
\begin{lemma}[Martingale difference bound]
    Let $t \in [n]$, and consider the setting of Proposition~\ref{dtvConcentrationProposition}.
    Then for all $s\in [t+1, ..., t+C(n)]$,
    \begin{align}
        |Z_s - Z_{s-1}| \leq 2\Delta,
    \end{align}
    where $\Delta$ is the constant in the specification of $\CC$ in Definition~\ref{ModelClassDefinition}.
    \label{MartingaleDifferenceLemma}
\end{lemma}
\begin{proof}
    We consider the following pair of dynamic mechanisms: $M_1$ and $M_1'$,
    which are coupled as follows: at times $1, ..., s-1$, they are equal.  At times $s$ and beyond,
    they evolve independently but according to $M_1$.
    We will show that we can rewrite the martingale differences in terms of $M_1$ and $M_1'$ as follows:
    \begin{align}
        Z_s - Z_{s-1}
        = \sum_{v=1}^t ( \Pr_{M_1}[\hmu_{t,v}=p, p_{t,v}=q ~|~ G_{s} ] 
                        - \Pr_{M_1'}[\hmu_{t,v}=p, p_{t,v}=q ~|~ G_{s} ]).
        \label{MartingaleDiffIdentity}
    \end{align}
    We will let $G_s$ be the graph at time $s$ according to $M_1$ and $G_s'$ be the 
    graph at time $s$ according to process $M_1'$.
    To prove the identity, by linearity of expectation, we have
    \begin{align}
        Z_s = \sum_{v=1}^t \Pr_{M_1}[ \hmu_{t,v}=p, p_{t,v}=q ~|~ G_{s}].
    \end{align}
    An expression for $Z_{s-1}$ can similarly be written, with the only change
    being in the conditioning: $G_{s-1}$ instead of $G_s$.  Now, conditioning
    on $G_{s}$ has the same effect on random variables from $M_1'$ as conditioning
    on $G_{s-1}$, since the choices made by $M_1'$ are independent of $G_s, G_{s+1}, ...$.
    This gives us (\ref{MartingaleDiffIdentity}).

    It remains to upper bound the 
    right-hand side of 
    (\ref{MartingaleDiffIdentity}).
    We note that
    \begin{align}
        \Pr_{M_1'}[ \hmu_{t,v}=p, p_{t,v}=q ~|~ G_{s}]
        = \E_{M_1'}[ \Pr[ \hmu_{t,v}=p, p_{t,v}=q ~|~ G_{s}, G'_s]].
    \end{align}
    I.e., we have conditioned on the choice at time $s$ according to $M_1'$.  We thus have
    \begin{align}
        &\Pr_{M_1}[ \hmu_{t,v}=p, p_{t,v}=q ~|~ G_{s}] - \Pr_{M_1'}[\hmu_{t,v}=p, p_{t,v}=q ~|~ G_s] \\
        &= \E_{M_1'}[  \Pr_{M_1}[ \hmu_{t,v}=p, p_{t,v}=q ~|~ G_{s}] - \Pr_{M_1'}[\hmu_{t,v}=p, p_{t,v}=q ~|~ G_s, G_s'] ].
    \end{align}
    The quantity inside the expectation measures the difference in probabilities
    after the choice at time $s$ is made according to each mechanism.  
    Using the conditional independence properties of the model class, at most $2\Delta$
    vertices $v$ are such that this difference is nonzero.  
    More precisely, let $W, W'$ be the sets of vertices $v$ in the two respective 
    trajectories such that $Y_{s,v} \neq Y_{s-1,v}$.  Since all future vertex choices
    relating to vertices not in these sets are conditionally independent of them given
    $G_s, G_s'$, the above difference is only nonzero for vertices in $W$ and $W'$.  By
    the model assumptions, their cardinalities are both at most $\Delta$.
    This completes the proof.
\end{proof}


    With Lemma~\ref{MartingaleDifferenceLemma} in hand, we apply the
    Azuma-Hoeffding inequality to prove Lemma~\ref{NpqConcentrationLemma}.
    The probability upper bound in Lemma~\ref{NpqConcentrationLemma} is given
    by
    \begin{align*}
        2\exp\left( -\frac{x^2}{\sum_{s=t+1}^{t+C(n)} (2\Delta)^2} \right)
        = 2\exp\left( -\frac{x^2}{(2\Delta)^2 C(n)} \right).
    \end{align*}

We are finally ready to prove 
Proposition~\ref{dtvConcentrationProposition}.
\begin{proof}[Proof of Proposition~\ref{dtvConcentrationProposition}]
    We will bound the integral representation 
    (\ref{dtvIntegralRepresentation}) of the total variation distance.
    
    We will start by ruling out large values of $q$.  In particular,
    set $q = \Omega(t^\delta)$, where $\delta = \frac{-1-\gamma-\epsilon}{\gamma}$, for an arbitrarily small positive
    $\epsilon$.  That is, we wish to upper bound the integral
    \begin{align}
        \left| \int_{0}^{1} \int_{q_0}^{1} N_t(p, q) - \E_{M_1}[N_{t}(p, q)] ~dq~dp \right|.
    \end{align}
    By (\ref{ImportantInequality}), we have
    \begin{align}
        \int_{q_0}^1 \E_{M_1}[N_t(p, q)] ~dq
        \leq \int_{q_0}^{1} \E_{M_1}[N_{M_0,t}(q)] ~dq
        = O((qt)^{1+\gamma}),
        \label{ApplicationOfImportantInequality}
    \end{align}
    where we recall that $\gamma < -2$.
    Similarly, by Markov's inequality, we have that for any $x > 0$,
    \begin{align}
        \Pr[ \int_{q=q_0}^1 N_t(p, q)~dq > x]
        \leq \Pr[\int_{q=q_0}^1 N_{M_0,t}(q)~dq > x]
        \leq \frac{\int_{q=q_0}^1 \E_{M_1}[N_{M_0,t}(q)]~dq}{x}
        \leq \frac{O((qt)^{1+\gamma})}{x}.
    \end{align}
    In particular, we will choose $x = \Theta((qt)^{1+\gamma+\epsilon})$,
    which has the following consequence:
    \begin{align}
        \Pr[ \int_{q=q_0}^1 N_t(p, q)~dq > x]
        = O((qt)^{-\epsilon}), 
    \end{align}
    and since $q=\Theta(t^{-\delta})$ with $\delta > -1$, this tends
    to $0$ polynomially fast in $t$.
    All of this has the consequence that 
    \begin{align}
        \Pr[ |\int_{0}^{1} \int_{q_0}^{1} N_t(p, q) - \E_{M_1}[N_{t}(p, q)] ~dq~dp | > (qt)^{1+\gamma+\epsilon} ]
        \leq O((qt)^{-\epsilon}).
    \end{align}
    
    We can thus ignore the range where $q = \Omega(t^{\frac{-1-\gamma-\epsilon}{\gamma}})$ for arbitrary fixed positive $\epsilon$.
    Now we focus on the small $q$ range: 
    $\delta \leq \frac{-1 - \gamma - \epsilon}{\gamma}$.  
    
    We will split the integral as follows:
    \begin{align}
        &\int_{0}^{t^{\frac{-1-\gamma-\epsilon}{\gamma}}} \int_{0}^1 |N_t(p, q) - \E_{M_1}[N_t(p, q)]| |p-q| ~dp~dq \\
        &= \int_{0}^{ t^{\frac{-1-\gamma-\epsilon}{\gamma}}} \int_{0}^{p_0}|N_t(p, q) - \E_{M_1}[N_t(p, q)]| |p-q| ~dp~dq \\
        &~~~+ \int_{0}^{ t^{\frac{-1-\gamma-\epsilon}{\gamma}} } \int_{p_0}^{1}|N_t(p, q) - \E_{M_1}[N_t(p, q)]| |p-q| ~dp~dq,
    \end{align}
    where $p_0 = (1+c)t^{\frac{-1-\gamma-\epsilon}{\gamma}}$, with $c$ some small positive constant.
    The first integral on the right-hand side can be handled by
    noticing that $|p-q| = O(t^{\frac{-1-\gamma-\epsilon}{\gamma}})$ 
    throughout, 
    and by Lemma~\ref{NpqConcentrationLemma} (with $x=t^{1/2+\eps_1}$), with probability
    exponentially close to $1$ (with respect to $t$), we have
    \begin{align}
        |N_t(p, q) - \E_{M_1}[N_t(p, q)]| \leq O(t^{1/2+\epsilon_1}).
    \end{align}
    Thus, the entire first integral is 
    $\leq O(t^{1/2 +\epsilon_1 - \frac{1 + \gamma + \epsilon}{\gamma}})$,
    which tends to $0$ because $\gamma < -2$.
    
    To handle the second integral (where $|p-q|$ may be $\Theta(1)$),
    we notice that $|p-q| \leq 1$, and then, with high probability
    (by Markov's inequality), $|N_t(p, q) - \E_{M_1}[N_t(p, q)]| \leq \E_{M_1}[N_t(p, q)]t^{\epsilon_2}$.  This last inequality can be seen more precisely as follows:
    since $N_t(p, q)$ is a cardinality, it is lower bounded by $0$.  Thus,
    if $N_t(p, q) < \E_{M_1}[N_t(p, q)]$, the difference between the two must
    be at most $\E_{M_1}[N_t(p, q)]$, which is certainly less than $\E_{M_1}[N_t(p, q)]t^{\epsilon_2}$.  On the other hand, we can see by Markov's inequality that
    \begin{align*}
        \Pr[N_t(p, q) > \E_{M_1}[N_t(p, q)] t^{\epsilon_2}]
        \leq \frac{ \E_{M_1}[N_t(p, q)]}   { \E_{M_1}[N_t(p, q)] t^{\epsilon_2} }
        = t^{-\epsilon_2}.
    \end{align*}
    
    We can further upper bound by $\E_{M_1}[N_t(p, q)] \leq \E_{M_1}[N_{\hmu_t}(p)]$ 
    (where $N_{\hmu_t}(p)$ is the number of vertices having conditional 
    probability $p$ according to the distribution $\hmu_{t}$), and since 
    $p \geq p_0$, 
    $\E_{M_1}[N_{\hmu_t}(p)] = O(t^{-\epsilon_3})$, which is a consequence of
    the model class definition.  In particular, the expected number of
    vertices $v$ with $p^{M_1}_{t,v} = p$ is at most $Ct^{-\epsilon}$, because of
    (\ref{ImportantInequality}) and the fact that $p \geq p_0$ (in exactly
    the same way that we concluded (\ref{ApplicationOfImportantInequality})).  This can be
    used to show that $\E_{M_1}[N_{\hmu_t}(p)] = O(t^{-\epsilon_3})$, using a proof analogous
    to the one that upper bounds the maximum
    degree for preferential attachment graphs.
    
    Then, we are left with the
    double integral
    \begin{align}
        t^{-\epsilon_3 + \epsilon_2} \cdot \int_{0}^{t^{ \frac{-1-\gamma-\epsilon}{\gamma}  }} \int_{0}^1 1 ~dq~dp 
        = O(t^{ \frac{-1-\gamma-\epsilon}{\gamma} + \epsilon_2}).
    \end{align}
    Recalling that $\gamma < -2$, we can make the exponent negative by
    setting $\epsilon$ and $\epsilon_2$ to be small enough, so the 
    right-hand side of the above expression tends to $0$.  More precisely,
    since $\epsilon > 0$, the exponent is at most $-\frac{1}{\gamma} - 1 + \epsilon_2$.  Since $\gamma < -2$, we have $-\frac{1}{\gamma} < 1/2$, so
    the exponent is less than $-1/2 + \epsilon_2$.  Thus, if we choose
    $\epsilon_2 < 1/2$, the resulting exponent is negative.
    This yields the claimed result.
\end{proof}

Now, we use Proposition~\ref{dtvConcentrationProposition} to complete the
proof of Theorem~\ref{MainTheorem}.  We shall show that the test
statistic is well concentrated.  In particular, we want to show the following.

\begin{theorem}[Concentration of the test statistic]
    We have, for any models $M_0, M_1 \in \CC$ (or $M_0$, $M_1$ as in
    Definition~\ref{GeneralModelClassDefinition}),
    and for any $b \in \{0, 1\}$ and any small positive constant $c$, that 
    \begin{align}
        \Pr_{M_b}[ |S_{M_0} - \E_{M_b}[S_{M_0}]| > cn] \leq O(n^{-\epsilon_1}),
    \end{align}
    for some $\epsilon_1$.
    \label{TestStatisticConcentrationTheorem}
\end{theorem}
\begin{proof}
    We note that the terms $S_j=d_{TV}(\hmu_{r_j}, p_{r_j}^{M_0})$ of the sum defining $S$ may be heavily 
    dependent, due to the fact that the sampling intervals are likely to 
    overlap heavily.  To circumvent this, define $X$ to be the number of
    probe points $r_j$ for which 
    \begin{align}
        |S_j - \E_{M_b}[S_j]| > c/2,
    \end{align}
    where we recall that $S_j$ is the term in the sum defining $S$ corresponding
    to the $j$th probe point.
    By Markov's inequality, for arbitrarily small $\epsilon_1 > 0$,
    \begin{align}
        \Pr[X > n^{1-\epsilon_1}]
        \leq \frac{\E[X]}{n^{1-\epsilon_1}}.
    \end{align}
    We can calculate $\E[X]$ using linearity of expectation, and we see from
    Proposition~\ref{dtvConcentrationProposition} that each term decays at
    least polynomially fast in $n$ to $0$.  We thus have
    \begin{align}
        \Pr[X > n^{1-\epsilon_1}]
        \leq n^{-\epsilon_2},
    \end{align}
    for some small $\epsilon_1, \epsilon_2 > 0$. 
    Now, if $X \leq n^{1-\epsilon_1}$, 
    then $|S - \E[S]| \leq O(n^{1-\epsilon_1}) + cn/2(1 - O(n^{-\epsilon_1}))$,
    which implies the desired result.
\end{proof}


Now, to finish the analysis of the error of the test, consider first the case
where the sample trajectory comes from $M_0$.  Then with probability
$1 - O(n^{-\epsilon_1})$, we have that $|S - \E_{M_0}[S_{M_0}]| < cn$
for arbitrary fixed $c > 0$,
so that the test correctly outputs $0$.

When the sample trajectory comes from $M_1$, note that for any small enough 
constant $\epsilon_3 > 0$, there is some $\epsilon_1 > 0$ such that with probability 
$1 - O(n^{-\epsilon_1})$, $|S - \E_{M_1}[S]| < \epsilon_3 n$.   
Then, by the condition in the theorem,
\begin{align}
    |\E_{M_1}[S] - \E_{M_0}[S]|
    \geq 
    d_n(M_0, M_1) - \E_{M_0}[S_{M_0}] - \E_{M_1}[S_{M_1}] \geq Dn,
\end{align}
where $D$ is the constant alluded to in the theorem statement.
Then, by the reverse triangle inequality, 
$|S - \E_{M_0}[S]| \geq \left| |\E_{M_0}[S] - \E_{M_1}[S]| - |\E_{M_1}[S] - \E_{M_0}[S]| \right| \geq Dn - \epsilon_3 n$, and for small enough $\epsilon_3$, this is $\geq Dn/2$.
Thus, $|S-\E_{M_0}[S]| > \alpha(n)$ with high probability.

Theorem~\ref{EstimationTheorem} is an immediate invocation of
Theorem~\ref{TestStatisticConcentrationTheorem}.

We note that the proof of concentration for $m \geq 1$ follows by generalizing
Lemma~\ref{NpqConcentrationLemma}.  In particular, this may be done by
induction on $m$, where $m=1$ forms the base case.  We introduce some 
convenient notation: let $N^{(j)}_{t}(p, q)$ denote the number of multisets 
$S$ 
of vertices from $[t-1]$ of cardinality at most $j$ whose empirical probability (according to the choices of the vertices in the interval
$[t, t+C(n)]$) is $p$ and for which $\prod_{v \in S} \pi_{t,v} = q$
(where the product takes into account the number of times each $v$ occurs in $S$).  In particular, $N^{(m)}_t(p, q) = N_t(p, q)$.

Now, we can write $N^{(m)}_t(p, q)$ by considering the conditional probability assigned to the first vertex chosen in teach timestep:
\begin{align}
    N^{(m)}_t(p, q)
    = \int_{p'\geq p} \int_{q' \geq q} N^{(1)}_t(p', q') N^{(m-1)}_t(p/p', q/q') ~dp'~dq'.
\end{align}
As above, we may neglect all but small values of $p'$ and $q'$ (on the order 
of $O(1/\sqrt{t})$), and the concentration of the remaining integral follows 
by induction.




\end{document}